\documentclass{article}
\usepackage{iclr2027_conference,times}
\iclrfinalcopy
\usepackage{amsmath,amssymb,amsthm}

\usepackage{amsmath,amsfonts,bm}

\def\eqref#1{equation~\ref{#1}}

\def\1{\bm{1}}

\DeclareMathAlphabet{\mathsfit}{\encodingdefault}{\sfdefault}{m}{sl}
\SetMathAlphabet{\mathsfit}{bold}{\encodingdefault}{\sfdefault}{bx}{n}

\usepackage{bm}

\usepackage{hyperref}

\usepackage{graphicx}

\usepackage{booktabs}
\usepackage{multirow}
\usepackage{array}

\usepackage[table]{xcolor}
\definecolor{oursbg}{RGB}{237,244,250}
\usepackage{pifont}

\usepackage{float}
\usepackage[ruled,vlined]{algorithm2e}
\usepackage[most]{tcolorbox}

\usepackage{xspace}
\newcommand{\method}{\textbf{$\bm{\pi}$PPO}\xspace}

\definecolor{piacblue}{HTML}{1F4E79}
\definecolor{piacpos}{HTML}{2E7D32}
\definecolor{piacneg}{HTML}{C62828}
\definecolor{piacgray}{HTML}{EEF3F8}
\definecolor{piacgold}{HTML}{8A6D1D}
\definecolor{rliableblue}{RGB}{0, 102, 204}

\newtheorem{proposition}{Proposition}
\newtcolorbox{thmbox}{
  enhanced, breakable,
  colback=piacblue!4!white,
  colframe=piacblue!85!black,
  boxrule=0.6pt, arc=1.2mm,
  left=9pt, right=9pt, top=7pt, bottom=7pt, boxsep=1pt,
  borderline west={2.2pt}{0pt}{piacblue},
}

\usepackage{enumitem}
\usepackage{capt-of}

\definecolor{privbg}{HTML}{FDF3E3}
\definecolor{privrule}{HTML}{D98F2B}
\definecolor{tokgray}{HTML}{9AA0A6}

\newcommand{\cmt}[1]{\textit{\color{gray}{/* #1 */}}}

\newcommand{\lbl}[1]{\textbf{#1}}

\newlist{pistep}{enumerate}{1}
\setlist[pistep]{label={\color{gray}\footnotesize(\arabic*)},
  leftmargin=2.1em, labelsep=.6em, topsep=0pt, itemsep=2pt, parsep=0pt}

\newtcolorbox{picol}[1]{enhanced, sharp corners, boxrule=.6pt,
  colback=white, colframe=black!55, coltitle=white, colbacktitle=black!68,
  fonttitle=\small, fontupper=\footnotesize, title={#1},
  left=6pt, right=6pt, top=6pt, bottom=6pt,
  before upper={\setlength{\parindent}{0pt}\setlength{\parskip}{3pt}}}

\newtcolorbox{privblock}{enhanced, sharp corners, breakable=false,
  colback=privbg, colframe=privrule, boxrule=0pt, leftrule=2.2pt,
  boxsep=0pt, left=6pt, right=6pt, top=5pt, bottom=5pt,
  before skip=4pt, after skip=4pt,
  before upper={\setlength{\parindent}{0pt}\setlength{\parskip}{3pt}}}

\usepackage[misc]{ifsym}

\makeatletter
\fancypagestyle{firstpage}{
  \fancyhf{}
  \renewcommand{\headrulewidth}{0.4pt}
  \IfFileExists{hunyuanlogo.png}{
    \lhead{\includegraphics[height=18pt]{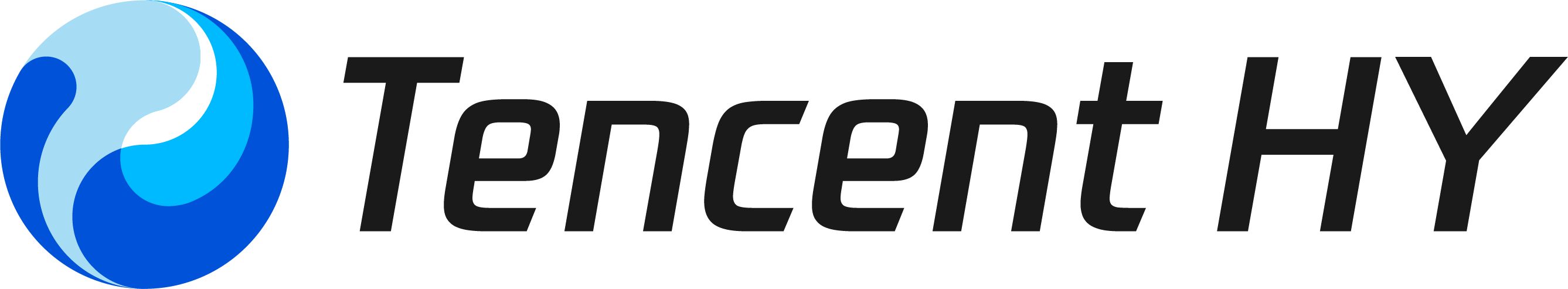}}
  }{}
  \renewcommand{\headrule}{{\color[HTML]{4D4D4D}\hrule\@height\headrulewidth\@width\headwidth\vskip-\headrulewidth}}
  \fancyfoot[C]{\thepage}
  \addtolength{\headsep}{-16pt}
  \addtolength{\footskip}{16pt}
}
\makeatother
\title{Privy to the Foil: Recasting Value Estimation with a Self-Privileged Critic for RLVR}

\author{
    \textbf{Kun Liang}$^{1,2,*}$ \quad
    \textbf{Chenming Tang}$^{1,2,*}$ \quad
    \textbf{Clive Bai}$^{3}$ \quad
    \textbf{Weijie Liu}$^{3}$  \\
    \textbf{Zeyuan Liu} $^{3}$  \quad
    \textbf{Qingyang Zhang} $^{3}$  \quad
    \textbf{Saiyong Yang}$^{3,\dagger}$ \quad
    \textbf{Yunfang Wu}$^{1,2,\dagger}$ \\
    \noalign{\vskip 5pt}
    $^1$School of Computer Science, Peking University \\
    $^2$National Key Laboratory for Multimedia Information Processing, Peking University \\
    $^3$Foundation Model Department, Tencent \\
    \noalign{\vskip 5pt}
    \Letter~{\fontsize{10.5}{12}\selectfont\texttt{kliang25@stu.pku.edu.cn}} \hspace{0.5em}
    \Letter~{\fontsize{10.5}{12}\selectfont\texttt{wuyf@pku.edu.cn}}
}

\begin{document}

\maketitle
\thispagestyle{firstpage}
\enlargethispage{16pt}
\vspace{-10pt}

\renewcommand*{\thefootnote}{\fnsymbol{footnote}}
\footnotetext{$^*$ Work done during internship at Tencent.}
\footnotetext{$^\dagger$ Corresponding Authors.}
\renewcommand*{\thefootnote}{\arabic{footnote}}

\begin{abstract}

Assigning credit to intermediate steps remains a central challenge in training Large Language Models (LLMs) on multi-step reasoning tasks with sparse terminal rewards, and actor-critic methods such as PPO address this by learning value functions to construct token-level advantages.
Their effectiveness, however, hinges on reliable value estimation, a difficult task requiring the critic to both assess progress toward a correct solution and anticipate an evolving policy's future behavior; errors in either can compromise credit assignment and destabilize online training.
In this paper, we revisit the standard state-only formulation of value estimation and propose \method, a self-privileged actor-critic framework.
By reusing verified same-prompt rollouts as contrastive evidence, $\pi$PPO helps the critic assess intermediate reasoning against successful and failed attempts, while preserving standard policy optimization and the deployment interface.
Experiments show that $\pi$PPO consistently improves value-estimation quality by a substantial margin and outperforms representative actor-critic and critic-free RLVR baselines on challenging mathematical reasoning benchmarks, while remaining effective even when paired with substantially smaller asymmetric critics.
\end{abstract}

\section{Introduction}

On-policy feedback-based learning methods stand out in LLM post-training for their effectiveness, especially on incentivizing long-form reasoning capabilities on verifiable tasks, providing a promising way for scaling~\citep{ds-r1,team2025kimi, qwen3}. 
Despite the empirical success of recent value-free methods, explicit value estimation remains a fundamental optimization paradigm in reinforcement learning. 
Beyond reducing variance, it provides a state-conditioned baseline, enabling principled credit assignment and readily accommodates diverse rollout structures~\citep{reinforce, greensmith2004variance,gae}.
Therefore, advancing value estimation offers a promising path toward improving policy optimization in LLM post-training.

However, accurately estimating the value of a partial response remains inherently challenging: the value network must both understand the space of correct solutions—the very task the policy itself is learning—and predict the policy’s future behavior~\citep{shan2026bringingvaluemodelsback,zhang2026v0generalistvaluemodel}. This difficulty is further exacerbated in long-form reasoning, where rewards are sparse and often observable only at the ends of exceptionally long trajectories, leaving a vast space of intermediate states with learning targets that are both weakly informative and highly uncertain. 
In actor–critic training, inaccurate value estimates can distort token-level advantages and misguide policy updates, undermining training stability and final performance~\citep{gae,vcppo}.

Existing approaches improve critic learning through better initialization and return estimation~\citep{vcppo,yue2025vapo, hou2026singlerolloutasynchronousoptimizationagentic}, or obtain more informative value targets through additional continuation sampling and intermediate-state modeling~\citep{kazemnejad2024vineppo,guo2025segmentpolicyoptimizationeffective, chen2026histanumcaestimatestate}.
Nevertheless, whether by strengthening the critic or enriching its targets, these solutions inherit the same largely unquestioned premise: both demands placed on the critic must be met by the critic itself, using only the information available from its prescribed vantage point; as we show later, however, this need not be taken as given.

To this end, we propose \textbf{P}rivileged-\textbf{I}nformation \textbf{PPO} (\method), which turns verified trajectories into online references for state-level value prediction.
We adapt PPO to sample groups of rollouts for each prompt, giving each trajectory a dual role: a sample for policy learning and a potential reference for evaluating partial responses in other trajectories.
For each target rollout, we construct a leave-target-out critic context from its verified siblings, preserving a valid baseline while leaving the policy's conditioning unchanged.
These references evolve with the policy and require no additional generation beyond the rollout groups used for policy updates.

As shown in Section~\ref{sec:exp}, $\pi$PPO consistently improves value-estimation quality across two model backbones on mathematical reasoning tasks, with corresponding gains in end-to-end policy performance over representative RL baselines. 
This design also enables $\pi$PPO to use substantially smaller critics while outperforming actor–critic baselines with actor-sized value models.

Our contributions are as follows:
\begin{itemize}
    \item \textbf{Conceptual.} We recast value estimation in RLVR as an information problem
    characterizing inherent properties of value prediction under binary verifiable rewards that highlight potential limitations of the conventional state-only formulation.
    \item \textbf{Methodological.} We propose $\pi$PPO, which supplies online evidence on the critic side, conditioning each value estimation on a leave-target-out context of verified same-prompt rollouts already collected for policy updates, thus preserving a valid baseline while leaving policy inference unchanged, all without additional rollout generation.
    \item \textbf{Empirical.} We show that $\pi$PPO consistently improves value-estimation quality and outperforms representative actor--critic and critic-free RLVR baselines on mathematical reasoning benchmarks. It remains effective even with substantially smaller asymmetric critics.
\end{itemize}

\section{Preliminary}

\paragraph{Notation.} 
In this paper, we define an LLM parameterized by $\theta$ as a policy $\pi_\theta$. Let $x$ denote a query and $\mathcal{D}$ the set of queries. Given a response $y$ to the query $x$, its likelihood under the policy $\pi_\theta$ is expressed as $\pi_\theta(y \mid x) = \prod_{t=1}^{|y|}\pi_\theta(y_t\mid x, y_{<t})$, where $|y|$ denotes the number of tokens in $y$. A query-response pair $(x,y)$ is scored by a rule-based outcome reward $r(x,y)\in \{0,1\}$, indicating whether the response $y$ aligns with the ground truth of $x$.

\paragraph{Proximal Policy Optimization (PPO).}
PPO~\citep{schulman2017ppo} follows the actor--critic framework, jointly learning a policy model (actor) $\pi_\theta$ and a value model (critic) $V_\phi$ parameterized by $\phi$, updating the policy using the following clipped policy surrogate (we omit the KL regularization term hereinafter for brevity):
{\small
\begin{align}
\mathcal{J}_\text{PPO}(\theta) = \mathbb{E}_{ x \sim \mathcal{D},\, y \sim \pi_{\theta_\text{old}}( \cdot | x) }
\left[ \frac{1}{|y|} \sum_{t=1}^{|y|} 
\min \left( \rho_{t}(\theta) \widehat{A}_{t},  \, \mathrm{clip} \left( \rho_{t}(\theta), 1 - {\varepsilon}, 1 + {\varepsilon}\right) \widehat{A}_{t} \right)
\right],
\end{align}
}
where
\(
\rho_t(\theta)
=
\frac{\pi_\theta(y_t\mid x,y_{<t})}
{\pi_{\theta_{\text{old}}}(y_t\mid x,y_{<t})}
\) denotes the token-level probability ratio at step $t$, and $\varepsilon$ is the clipping coefficient. The learned critic enters the policy update through the advantage \(\widehat A_t\), which is commonly constructed using Generalized Advantage Estimation (GAE)~\citep{gae}. Since the critic scores every state \(s_t=(x,y_{<t})\) individually, the objective thereby enables dense, token-level credit assignment even when the environment returns a single terminal reward. 

\paragraph{Group Relative Policy Optimization (GRPO).}
GRPO~\citep{shao2024grpo} bypasses the value model by computing the relative advantage of each response within a group of $G$ responses to the same query.
Given $\{y_i\}_{i=1}^{G}\sim
\pi_{\theta_{\mathrm{old}}}(\cdot\mid x)$, it averages the clipped surrogate over the group, using
{\small
\begin{equation}
\widehat A_{i,t}=\widehat A_i
=\frac{
r(x,y_i)-\mathrm{mean}\big(\{r(x,y_j)\}_{j=1}^{G}\big)
}{
\mathrm{std}\big(\{r(x,y_j)\}_{j=1}^{G}\big)
}.
\label{eq: adv-ratio}
\end{equation}
}
All tokens in $y_i$ share the same advantage $\widehat A_i$.

\paragraph{Privileged Information (PI).}
Learning using privileged information augments a standard input \(x\in\mathcal X\) with task-relevant auxiliary information \(x^\star\in\mathcal X^\star\) available during training but absent from the standard inference input~\citep{vapnik2009new}. For LLMs, both the query and privileged information can be represented as token sequences over the model vocabulary \(\mathcal V\), i.e., \(x,\mathcal I\in\mathcal V^\ast\), and incorporated by conditioning the policy as \(\pi_\theta(\cdot\mid x,\mathcal I)\), where \(\mathcal I\) may encode hints, verified answers, or reference rationales~\citep{qu2026pope,zhao2026self}.

\section{Methodology}

\label{sec:method}
In this section, we introduce a new perspective on incorporating privileged information into policy optimization: enriching the critic’s information while preserving the policy’s original objective of maximizing environment-defined rewards. We instantiate this perspective using information available within the training process, without additional generation or external expert knowledge.

\subsection{Interpreting a state: the ``missing'' context in value estimation}

Standard critics fit a single value predictor across states, predicting each state's expected return from its representation. We view this shared-prediction setup through the lens of fully amortized optimization~\citep{shu2017amortized, amos2025tutorialamortizedoptimization}, in which a generic model serves all prediction instances without instance-specific adaptation or auxiliary evidence. We first examine the resulting limitation of this formulation and then characterize the underlying information structure for value estimation under binary verifiable rewards.

\begin{thmbox}
\begin{proposition}[Privileged conditioning for amortized value estimation]
\label{prop:gap}
\leavevmode\par\nobreak\vspace{1pt}
Fix a policy $\pi$ with undiscounted terminal return $R\in\{0,1\}$,
and define $V^\pi(s)=\mathbb{E}[R\mid s]$ and
$V^{\pi,+}(s,\mathcal I)=\mathbb{E}[R\mid s,\mathcal I]$.
Assume finite squared error and that contextual critics
$V_\phi^+(s,\mathcal I)$ can represent every state-only critic
$V_\phi(s)$. Under a common joint distribution of $(s,\mathcal I,R)$,
\begingroup
\fontsize{9.5pt}{9pt}\selectfont
\[
\begin{aligned}
    \underbrace{
    \inf_\phi \mathbb{E}[(R-V_\phi^+(s,\mathcal I))^2] }_{\text{contextual prediction risk}} &= \mathbb{E}_{s,\mathcal I}[\operatorname{Var}(R\mid s,\mathcal I)]
    + \inf_\phi \mathbb{E}[(V_\phi^+(s,\mathcal I)-V^{\pi,+}(s,\mathcal I))^2]
    \\[2pt]
    &\le  \mathbb{E}_s[\operatorname{Var}(R\mid s)]
    + \inf_\phi \mathbb{E}[(V_\phi(s)-V^\pi(s))^2] = \underbrace{
    \inf_\phi \mathbb{E}[(R-V_\phi(s))^2]
}_{\text{state-only prediction risk}}.
\end{aligned}
\]
\endgroup
\end{proposition}
\end{thmbox}

Thus, the residual error is not solely an optimization artifact: part of it reflects return uncertainty unresolved by the target state alone. 
We interpret this gap as \emph{informational}: a state-conditioned predictor captures only
the return variation predictable from \(s\), while additional evidence \(\mathcal I\) can reduce the Bayes error insofar as it explains residual variation beyond \(s\).

Under binary verifiable rewards, this return uncertainty concerns whether a trajectory will succeed or fail. We next examine value prediction through this outcome distinction, relating prediction error to the contrast between the critic's predictions on successful and failed trajectories.

\begin{thmbox}
\begin{proposition}[Success--failure contrast in value estimation]
\label{prop:contrast}
\leavevmode\par\nobreak\vspace{1pt}
Fix a policy $\pi$ with undiscounted terminal return $R\in\{0,1\}$
and $\operatorname{Var}(R)>0$.
All expectations are taken under the common joint distribution of $(s,R)$.
For any square-integrable $V_\phi(s)$, define
\[
    \Delta_\phi = \mathbb{E}[V_\phi(s)\mid R=1] - \mathbb{E}[V_\phi(s)\mid R=0].
\]
Its prediction risk satisfies
\[
    \mathbb{E}[(R-V_\phi(s))^2] \ge
    \underbrace{\operatorname{Var}(R)(1-\Delta_\phi)^2 }_{\text{success--failure contrast error}}.
\]
Consequently, any critic outperforming the optimal constant predictor
$V_0=\mathbb{E}[R]$ (with risk $\operatorname{Var}(R)$)
must exhibit positive success--failure contrast, i.e., $\Delta_\phi>0$.
\end{proposition}
\end{thmbox}

The bound specifies a necessary condition on the contrast between the critic's predictions on successful and failed trajectories for achieving a given level of prediction accuracy.

Taken together, we therefore rethink the evidence also as \emph{contrastive}: 
by providing positive and negative references for the outcome distinction that the critic must infer,
we reduce the uncertainty left by state-only value estimation insofar as the evidence
clarifies the target prefix's relative support under the successful and failed trajectory processes.

\subsection{\texorpdfstring{$\bm{\pi}$PPO}{piPPO}}

\begin{figure*}[t]
\centering
\includegraphics[width=0.994\linewidth]{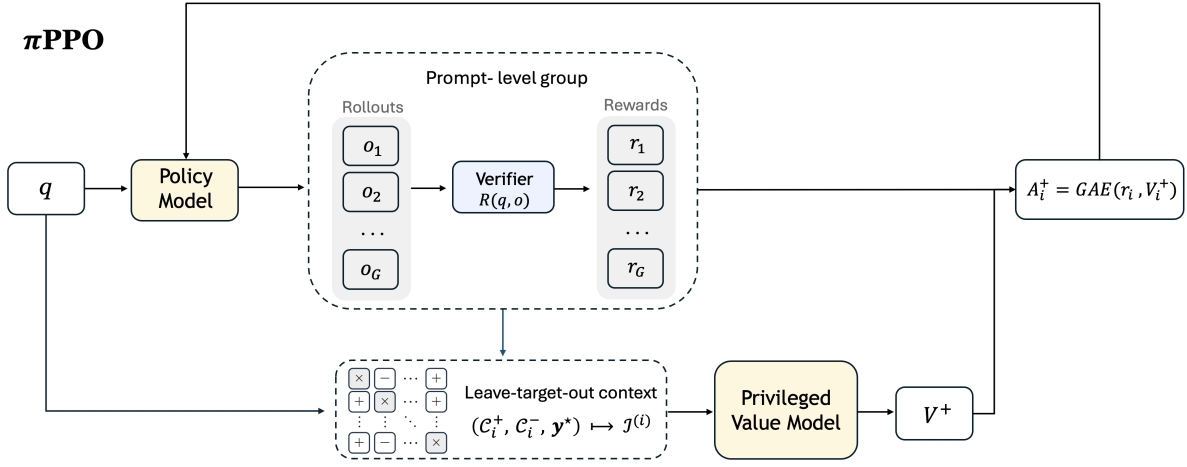}
\caption{
Overview of $\pi$PPO. For each prompt, the policy generates a group of rollouts, and the verifier assigns rewards according to their outcomes. The labeled rollouts are then reused to construct privileged context $\mathcal I^{(i)}$ for each target rollout, excluding the target itself. The privileged value model uses this context to estimate values along the target rollout. These estimates, together with the verifier rewards, are used to compute advantages for the PPO policy update.
}
\label{fig:overview}
\end{figure*}

Motivated by the analysis above, we propose $\pi$PPO (Figure~\ref{fig:overview}), a critic-side extension of PPO that reuses verified same-prompt rollouts as contrastive evidence for value estimation. 
By constructing a leave-target-out context pool via an adaptation of PPO to group rollouts, we enable each trajectory to serve both as a training target to be evaluated and as a potential labeled reference for its siblings, allowing the critic to assess a prefix in relation to successful and failed attempts at the same problem.

\textbf{Leave-target-out contrastive context.} For a prompt $\boldsymbol x$, the old policy samples $G$ rollouts $\boldsymbol y^{(j)}\sim\pi_{\theta_{\mathrm{old}}}(\cdot\mid\boldsymbol x)$, which the verifier labels. The context of a target $\boldsymbol y^{(i)}$ is drawn from the remaining rollouts, partitioned by label,
\begin{equation}
\mathcal C_i^{\sigma}
=\bigl\{\boldsymbol y^{(j)}: j\neq i,\;
r(\boldsymbol{x},\boldsymbol{y}^{(j)})=\mathbf 1_{\{\sigma=+\}}\bigr\},
\qquad \sigma\in\{+,-\}.
\label{eq:sibling-sets}
\end{equation}
To use trajectories as evidence for prefix-level prediction, we separate the target being evaluated from the trajectories supplying its context.
Let $\mathcal S_k(\mathcal C)$ denote $k$ reference rollouts randomly sampled from $\mathcal C$, together with their verifier labels.
Writing $\Vert$ for concatenation and $\boldsymbol y^\star$ for the supplied ground-truth answer to $\boldsymbol x$, we construct
\begin{equation}
\mathcal I^{(i)}=
\begin{cases}
\mathcal S_1(\mathcal C_i^+)\Vert\mathcal S_1(\mathcal C_i^-),
&\mathcal C_i^+\neq\varnothing,\ \mathcal C_i^-\neq\varnothing,\\
\mathcal S_2(\mathcal C_i^+),
&\mathcal C_i^-=\varnothing,\\
\mathcal S_2(\mathcal C_i^-)\Vert[\boldsymbol y^\star],
&\mathcal C_i^+=\varnothing.
\end{cases}
\label{eq:piac-context}
\end{equation}
When only one outcome is represented among the remaining rollouts, we use two siblings of that outcome; the ground-truth answer is additionally supplied only in the all-negative case and is identified separately from the sampled reference attempts.

\textbf{Privileged advantage estimation.}
By constructing the self-elicited privileged information above, we obtain a critic that no longer predicts from the prefix alone but performs a contextual forward pass $V_\phi^{+}(\boldsymbol x,\mathcal I^{(i)},\boldsymbol y_{<t}^{(i)})$, while the actor stays conditioned on $(\boldsymbol x,\boldsymbol y_{<t}^{(i)})$ alone.
These context-conditioned values enter the standard TD residuals and GAE:
\begin{equation}
\begin{aligned}
\delta_t^{+(i)}
&=r_t^{(i)}
+\gamma V_\phi^+(s_{t+1}^{(i)},
\textcolor{blue!65!black}{\mathcal I^{(i)}})
-V_\phi^+(s_t^{(i)},
\textcolor{blue!65!black}{\mathcal I^{(i)}}),\\
\widehat A_t^{+(i)}
&=\sum_{\ell=0}^{T_i-t}
(\gamma\lambda)^\ell\delta_{t+\ell}^{+(i)}.
\end{aligned}
\label{eq:privileged-gae}
\end{equation}
Here $T_i=|\boldsymbol y^{(i)}|$, and the terminal value is set to zero. The actor is then updated by substituting $\widehat A_t^{+(i)}$ for the advantage in the PPO clipped objective, while retaining its original conditioning.

Notably, the $\pi$PPO design yields two favorable properties:
\begin{itemize}
    \item \textbf{Valid policy gradient baseline under privileged value estimation.}
    Let \(s_t=(\boldsymbol{x},\boldsymbol{y}_{<t})\), and let the privileged
    information \(\mathcal{I}\) be available only to the critic. Because
    \(V_\phi^{+}(s_t,\mathcal{I})\) is independent of the current action \(a_t\),
    it remains a valid baseline:
    \begin{equation}
    \mathbb{E}_{a_t\sim\pi_\theta(\cdot\mid s_t)}
    \left[
    \nabla_\theta\log\pi_\theta(a_t\mid s_t)
    V_\phi^{+}(s_t,\mathcal{I})
    \right]
    =
    V_\phi^{+}(s_t,\mathcal{I})
    \nabla_\theta
    \sum_{a_t}\pi_\theta(a_t\mid s_t)
    =0.
    \end{equation}
    Thus, privileged value estimation introduces no additional bias into the
    expected actor gradient and leaves the underlying policy objective unchanged.

    \item \textbf{Training--inference policy consistency.}
    The privileged context enters only the critic, so the actor uses
    the same conditioning at both training and inference:
    \begin{equation}
    \pi_\theta^{\mathrm{train}}(a_t\mid s_t)
    =\pi_\theta^{\mathrm{infer}}(a_t\mid s_t)
    =\pi_\theta(a_t\mid s_t),
    \qquad s_t=(\boldsymbol x,\boldsymbol y_{<t}).
    \end{equation}
\end{itemize}

Overall, $\pi$PPO incorporates verified contrastive rollout evidence into policy optimization through context-conditioned value estimates and resulting advantages. Since the rollout group collected for each PPO update supplies the critic context, the evidence is refreshed on the fly as the policy evolves.

\section{Experiments}
\label{sec:exp}
\subsection{Experimental Settings}

In this section, we provide a brief overview of the key experimental setup, including training procedures, baselines,
and evaluation details. Additional information can be found in Appendix~\ref{app:implementation}.

\textbf{Training details.} 
We primarily conduct experiments on Qwen3-4B and Qwen3-8B~\citep{qwen3} with thinking mode off, using DAPO-17K~\citep{yu2026dapo}. We build on the VeRL codebase~\citep{verl} and follow the standard PPO recipe with clip-higher enabled. All methods use the same hyperparameters (batchsize 256, actor learning rate 1e-6) and rollout settings (temperature 1.0, top-p 1.0, top-k disabled, 8 rollouts per prompt). Following~\citet{vcppo}, we adopted 50 steps pretrain for the value model warmup as adjustment for RLVR tasks for all critic-based methods, further implementation details can be found in the Appendix~\ref{app:train}.

\paragraph{Baselines.}
We compare \(\pi\)PPO against representative RL baselines. We use vanilla PPO~\citep{schulman2017ppo} as the standard actor--critic baseline and include VAPO~\citep{yue2025vapo} as an RLVR-adapted improvement over PPO. We further compare against GRPO~\citep{shao2024grpo} and DAPO~\citep{yu2026dapo}, a widely adopted GRPO variant with competitive performance and robust training on RLVR tasks.
Together, these experiments assess the benefit of augmenting the PPO critic with verified sibling context and compare $\pi$PPO with RLVR-adapted actor--critic and strong value-free baselines on reasoning tasks. We additionally examine the training behavior of the self-distillation family in Appendix~\ref{app:opsd}.

\paragraph{Benchmarks and evaluation metrics.}
We evaluate our method on challenging in-domain math benchmarks: AIME 2024~\citep{aime2024}, AIME 2025~\citep{aime2025}, BeyondAIME~\citep{bytedance_seed_2025_beyondaime}, HMMT25 (February)~\citep{hmmt25}, and HMMT25 (November)~\citep{hmmt25}. Following~\citet{ds-r1}, we generate 8--16 outputs per problem, depending on test-set size, and report $Pass@1$ accuracy. All methods are evaluated using identical prompts and decoding configurations; details are provided in Appendix~\ref{app:evaluation}. We further assess out-of-distribution generalization, with results reported in Appendix~\ref{app:ood}.

For assessing the critic, we mainly focus on explained variance (EV) and value-function loss. Let $\widehat{G}_{i,t}$ and $\widehat{V}_{i,t}$ denote the value-regression target and critic prediction at token position $t$ in rollout $i$, respectively. Over the valid response-token positions
$\mathcal{M}$, we compute
\begin{equation}
\operatorname{EV}
=
1-
\frac{
    \operatorname{Var}_{(i,t)\in\mathcal{M}}
    \left(\widehat{G}_{i,t}-\widehat{V}_{i,t}\right)
}{
    \operatorname{Var}_{(i,t)\in\mathcal{M}}
    \left(\widehat{G}_{i,t}\right)+\epsilon
},
\qquad
\mathcal{L}_{V}
=
\frac{1}{|\mathcal{M}|}
\sum_{(i,t)\in\mathcal{M}}
\left(\widehat{G}_{i,t}-\widehat{V}_{i,t}\right)^2.
\label{eq:critic-metrics}
\end{equation}
Together, EV and $\mathcal{L}_{V}$ provide complementary views of across-sample target-variance capture and pointwise regression-target fitting, respectively.

\subsection{Main Results}

\begingroup
\setlength{\intextsep}{5pt plus 1pt minus 1pt}
\setlength{\abovecaptionskip}{5pt}

\makeatletter
\@ifundefined{bcol}{\newlength{\bcol}}{}
\@ifundefined{ccol}{\newlength{\ccol}}{}
\@ifundefined{scol}{\newlength{\scol}}{}
\@ifundefined{\string\color@oursbg}{\definecolor{oursbg}{RGB}{237,243,254}}{}
\makeatother
\setlength{\bcol}{50pt}
\setlength{\ccol}{30pt}
\setlength{\scol}{32pt}
\providecommand{\bname}[1]{\mbox{\footnotesize\bfseries #1}}
\providecommand{\kavg}{}
\renewcommand{\kavg}[1]{\mbox{\fontsize{8.2}{9}\selectfont\texttt{#1}}}
\providecommand{\benchhead}{}
\renewcommand{\benchhead}[2]{%
  \begin{tabular}{@{}c@{}}
    \makebox[\bcol][c]{\bname{#1}}\\[-0.8pt]
    \makebox[\bcol][c]{\kavg{#2}}%
  \end{tabular}%
}
\providecommand{\costhead}{}
\renewcommand{\costhead}[2]{%
  \begin{tabular}{@{}c@{}}
    \makebox[\ccol][c]{\bname{#1}}\\[-0.8pt]
    \makebox[\ccol][c]{\kavg{#2}}%
  \end{tabular}%
}
\providecommand{\benchheadS}{}
\renewcommand{\benchheadS}[3]{%
  \begin{tabular}{@{}c@{}}
    \makebox[\scol][c]{\bname{#1}}\\[-1.6pt]
    \makebox[\scol][c]{\bname{#2}}\\[-0.8pt]
    \makebox[\scol][c]{\kavg{#3}}%
  \end{tabular}%
}
\providecommand{\costheadS}{}
\renewcommand{\costheadS}[2]{%
  \begin{tabular}{@{}c@{}}
    \makebox[\ccol][c]{\bname{#1}}\\[-1.6pt]
    \makebox[\ccol][c]{\bname{\strut}}\\[-0.8pt]
    \makebox[\ccol][c]{\kavg{#2}}%
  \end{tabular}%
}

\begin{table}[H]
\centering
\caption{Mathematical reasoning accuracy (\%).
``\texttt{Avg@k}'' denotes mean accuracy over $k$ random generations (i.e., pass@1);
\textbf{Overall} is the mean over the five benchmarks.
Best result within each model size is in bold.
}
\label{tab:main}

\small
\setlength{\tabcolsep}{2pt}
\renewcommand{\arraystretch}{1.12}
\setlength{\ccol}{50pt}

\resizebox{\linewidth}{!}{%
\begin{tabular}{l *{6}{w{c}{\bcol}} w{c}{\ccol}}
    \toprule

    \multirow{2}{*}{\textbf{Method}}
    & \multicolumn{6}{c}{\textbf{Mathematical Reasoning}}
    & \multicolumn{1}{c}{\textbf{Tokens}} \\
    \cmidrule(lr){2-7} \cmidrule(lr){8-8}
    & \benchhead{AIME 24}{Avg@16}
    & \benchhead{AIME 25}{Avg@16}
    & \benchhead{Beyond AIME}{Avg@8}
    & \benchhead{HMMT\textsubscript{Feb}}{Avg@16}
    & \benchhead{HMMT\textsubscript{Nov}}{Avg@16}
    & \benchhead{Overall}{Avg.}
    & \costhead{Gen./ Train}{$10^{9}$} \\
    \midrule

        \multicolumn{8}{@{}l@{}}{\textbf{(a) Qwen3-4B}} \\

        \quad Initial policy
        & 21.7 & 18.9 & 10.8 & 12.3 & 7.3 & 14.2 & -- \\

        \addlinespace[1pt]
        \multicolumn{8}{@{}l@{}}{\emph{Critic-free}} \\

        \quad GRPO
        & 60.4 & 54.8 & 34.5 & 34.6 & 41.5 & 45.2 & 2.33\,/\,2.38 \\

        \quad DAPO
        & 63.1 & 56.7 & 34.5 & 34.6 & 41.9 & 46.2 & 4.42\,/\,2.64 \\

        \addlinespace[1pt]
        \multicolumn{8}{@{}l@{}}{\emph{Actor--critic}} \\

        \quad PPO
        & 57.9 & 54.2 & 32.4 & 32.7 & 40.4 & 43.5 & 2.42\,/\,2.48 \\

        \quad VAPO
        & 63.3 & 52.1 & 33.3 & 33.8 & 43.9 & 45.3 & 2.07\,/\,2.14 \\

        \rowcolor{oursbg}
        \quad \textbf{$\bm{\pi}$PPO (ours)}
        & \textbf{66.7} & \textbf{61.5} & \textbf{39.0}
        & \textbf{37.5} & \textbf{46.9} & \textbf{50.3}
        & 2.53\,/\,2.59 \\

        \addlinespace[2pt]
        \midrule
        \multicolumn{8}{@{}l@{}}{\textbf{(b) Qwen3-8B}} \\

        \quad Initial policy
        & 25.0 & 20.2 & 13.1 & 11.5 & 10.0 & 16.0 & -- \\

        \addlinespace[1pt]
        \multicolumn{8}{@{}l@{}}{\emph{Critic-free}} \\

        \quad GRPO
        & 69.4 & 58.3 & 36.1 & 32.5 & 45.2 & 48.3 & 2.25\,/\,2.30 \\

        \quad DAPO
        & 71.3 & 57.1 & 37.8 & 33.3 & 46.9 & 49.3 & 4.34\,/\,2.59 \\

        \addlinespace[1pt]
        \multicolumn{8}{@{}l@{}}{\emph{Actor--critic}} \\

        \quad PPO
        & 68.1 & 55.2 & 36.3 & 32.3 & 47.7 & 47.9 & 2.35\,/\,2.42 \\

        \quad VAPO
        & 67.3 & 55.8 & 34.3 & 33.5 & 47.1 & 47.6 & 1.95\,/\,2.02 \\

        \rowcolor{oursbg}
        \quad \textbf{$\bm{\pi}$PPO (ours)}
        & \textbf{73.5} & \textbf{61.5} & \textbf{40.0}
        & \textbf{34.4} & \textbf{48.5} & \textbf{51.6}
        & 2.57\,/\,2.63 \\

        \addlinespace[2pt]
        \bottomrule
\end{tabular}%
}
\end{table}
\begin{figure}[H]
    \centering
    \includegraphics[width=0.994\linewidth]{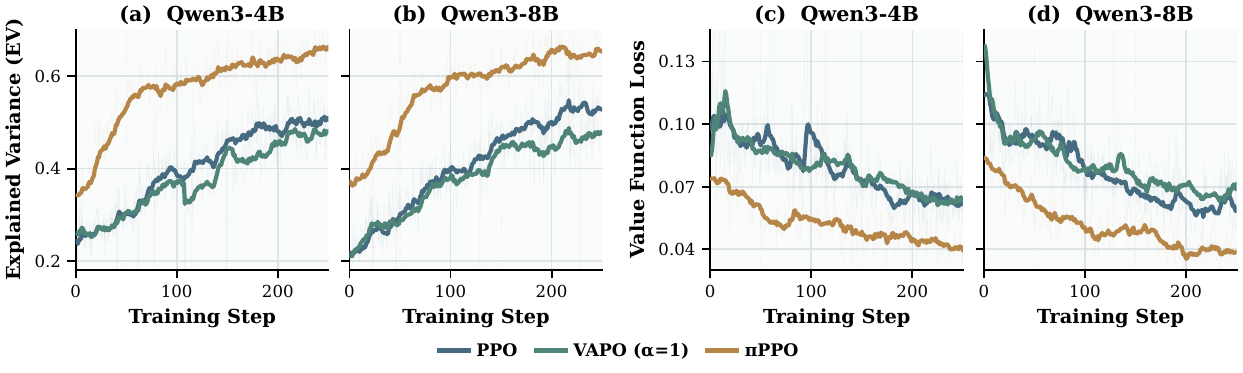}
    \caption{Critic explained variance (EV; a,b) and value-function loss (c,d) during training. Across Qwen3-4B and Qwen3-8B, $\pi$PPO consistently achieves higher EV and lower value-function loss than PPO and VAPO, with improvements evident from the early training steps.}
    \label{fig:critic-quality}
\end{figure}

\paragraph{\ding{182} $\pi$PPO consistently improves end-to-end policy performance.}
As shown in Table~\ref{tab:main}, $\pi$PPO achieves the highest accuracy on all five mathematical reasoning benchmarks at both model scales. Its overall scores reach 50.3\% on Qwen3-4B and 51.6\% on Qwen3-8B, exceeding the strongest baseline, DAPO, by 4.1 and 2.3 percentage points, respectively. Compared with vanilla PPO, the gains are 6.8 and 3.7 points. These results show that privileged value estimation translates into consistent policy gains over both actor--critic and critic-free baselines.

\textbf{\ding{183} Privileged context substantially improves value estimation.} Figure~\ref{fig:critic-quality} compares the critic explained variance and value-function loss throughout training. With either backbone model, the two metrics reveal a consistent pattern: $\pi$PPO starts with a higher explained variance upon entering training, improves sharply during the early updates, and subsequently remains at a high level, while its value-function loss quickly decreases and stays consistently lower. In contrast, the state-only critics of PPO and VAPO improve more gradually and plateau at weaker levels, providing less informative value estimates throughout training.
\endgroup

\subsection{$\pi$PPO admits smaller asymmetric critics}
\label{sec:asymmetric_critic}

Another commitment of standard actor--critic implementations is to instantiate the actor and critic with the same model backbone.
We next ask whether the value estimation advantage in $\pi$PPO can be traded for critic capacity.
To test this hypothesis, we pair Qwen3-4B and Qwen3-8B actors with their corresponding 0.6B and 1.7B critics from the same family. For both PPO and $\pi$PPO, we compare these asymmetric configurations against their actor-sized symmetric counterparts.

\begin{figure}[H]
    \centering
    \includegraphics[width=0.994\linewidth]{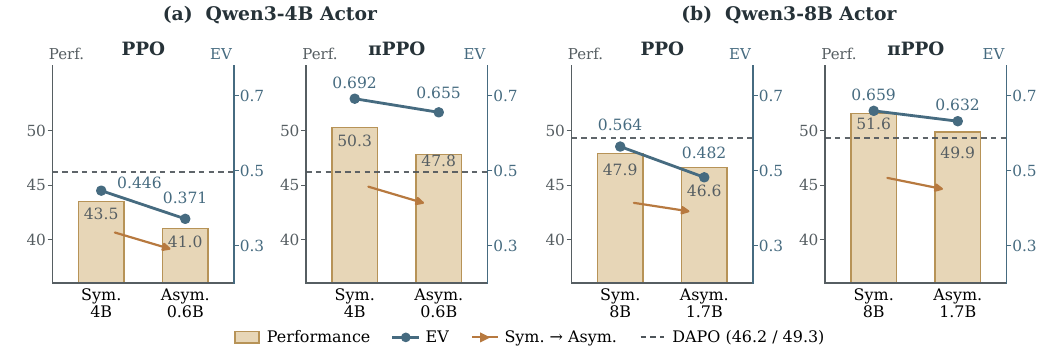}
    \caption{Performance (bars, left axis) and explained variance (EV; curves, right axis) of PPO and $\pi$PPO. Sym.\ and Asym.\ denote critics matching the actor size and smaller critics, respectively. Dashed lines indicate DAPO, the best-performing baseline.}
    \label{fig:asym}
\end{figure}

\textbf{\ding{182} $\pi$PPO achieves strong overall performance with substantially smaller critics.}
As shown in Figure~\ref{fig:asym}, $\pi$PPO achieves overall accuracies of 47.8\% and 49.9\% under asymmetric actor--critic configurations, still exceeding the strongest baseline on both backbones. Notably, it also outperforms actor--critic baselines with substantially larger, actor-sized critics, showing that privileged evidence can unlock stronger policy optimization with compact critics. 

\textbf{\ding{183} $\pi$PPO retains its value estimation advantage at a fraction of the critic size.}
Reducing the critics to 0.6B and 1.7B cuts their parameter counts by approximately 85\% and 79\%, yet EV decreases by only 0.037 and 0.027. These declines are smaller than in PPO, and the resulting EVs of 0.655 and 0.632 still exceed the 0.446 and 0.564 achieved by PPO with actor-sized critics.
Overall accuracy decreases by only 2.5 and 1.7 percentage points relative to the symmetric configurations.

Taken together, these results suggest that privileged value estimation offers a rationale for revisiting actor--critic scale matching: the improved estimation quality motivates asymmetric designs that retain strong policy performance with substantially smaller critics.
\section{Analysis}

\subsection{Characterizing Privileged Value Estimation}

To investigate the performance gap between PPO and $\pi$PPO, we examine their critic predictions, which provide advantage estimates for an otherwise unchanged PPO update.
All analyses use the training data.
Since ground-truth state values are not directly observable, we construct an empirical reference by running 256 MC samples for each prefix state $s_t$ and averaging the returns.
We compare both critics against this MC reference to assess the accuracy of their value predictions.

\begingroup
\setlength{\abovecaptionskip}{5pt}
\textbf{Accuracy of value estimates.} We first group the prefix states by their MC-estimated values and compare the distributions of predictions from the two critics in Figure~\ref{fig:value_acc}(a). Accurate estimates should concentrate near the diagonal, whereas systematic deviations indicate over- or underestimation. PPO exhibits widely dispersed predictions, particularly intermediate value states, while $\pi$PPO generally produces tighter distributions that more closely follow the MC reference.
Figure~\ref{fig:value_acc}(b) further shows that $\pi$PPO achieves lower MSE at every evaluated prefix position. As generation progresses, PPO's error generally increases, whereas that of $\pi$PPO remains comparatively stable. At the final evaluated position, the MSE is approximately 0.05 for $\pi$PPO versus 0.15 for PPO.

\begin{figure}[H]
    \centering
    \includegraphics[width=0.994\linewidth]{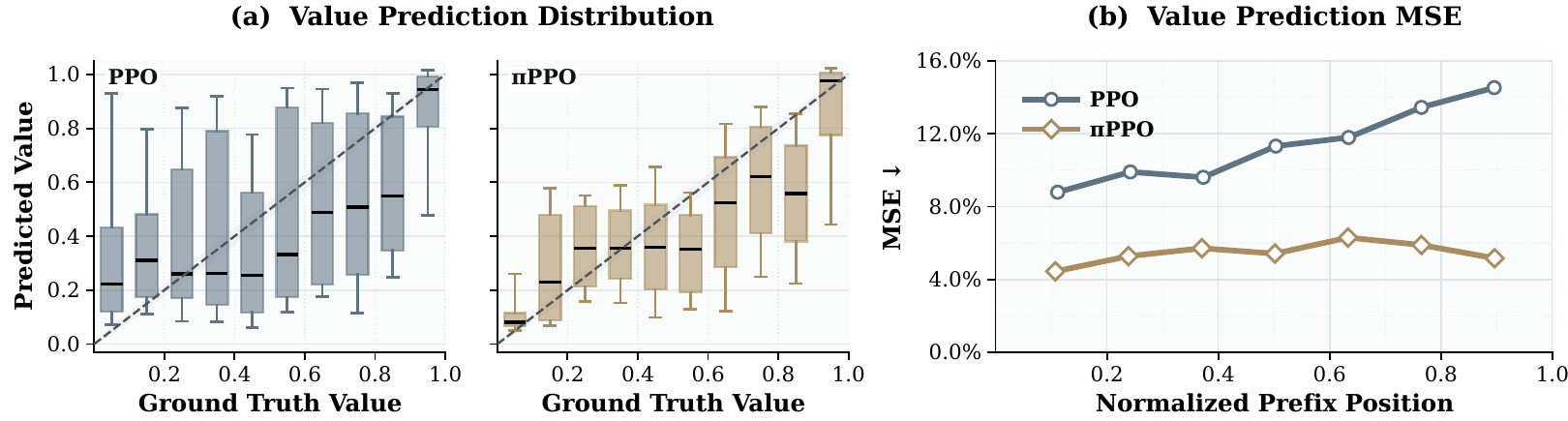}
    \caption{\textbf{(a)} Critic predictions grouped by MC reference value. Boxes show the interquartile range, black bars indicate the mean, and whiskers mark the 10th and 90th percentiles. The dashed diagonal denotes exact agreement with the MC reference. \textbf{(b)} Mean squared error against the MC reference across normalized prefix positions.}
    \label{fig:value_acc}
\end{figure}
\begin{figure}[H]
\centering
    \includegraphics[width=0.994\linewidth]{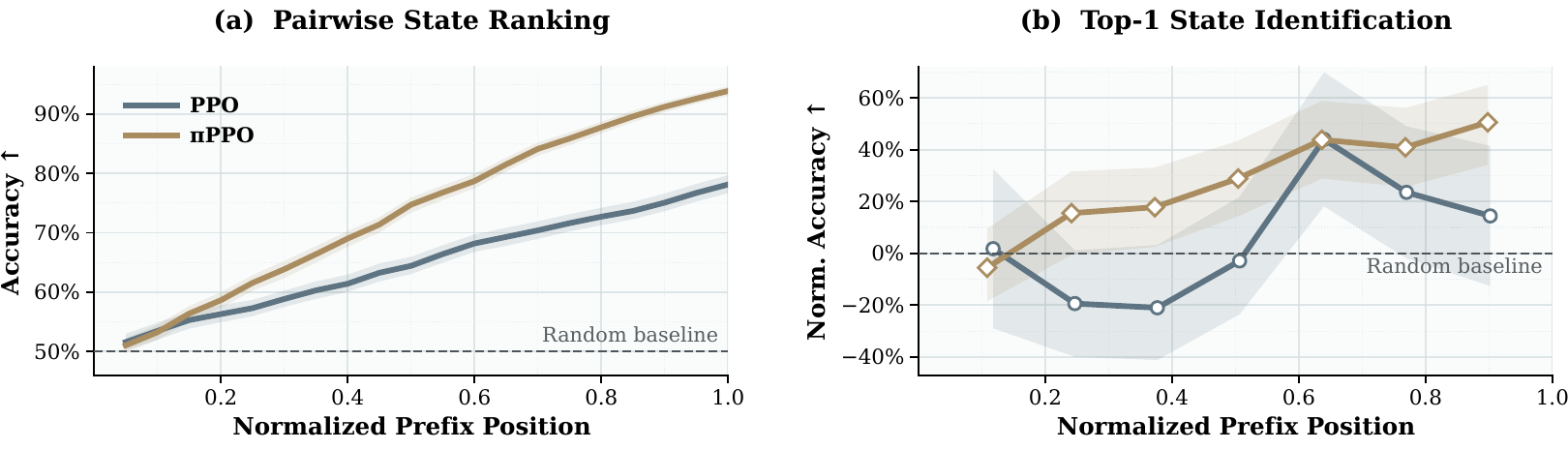}
    \caption{Critic ranking and state identification across normalized prefix positions. \textbf{(a)} Probability of assigning a higher value to a prefix from a successful trajectory than to one from an unsuccessful trajectory. \textbf{(b)} Chance-normalized top-1 accuracy for identifying the state with the highest MC reference value among prefixes sharing the same prompt and reasoning-node rank.
    Dashed lines indicate random-selection baselines.
    }
\label{fig:ranking-identification}
\end{figure}

\textbf{State ranking and identification.}
Pairwise ranking measures how often a critic assigns a higher value to a prefix from a successful trajectory than to one from an unsuccessful trajectory.
Figure~\ref{fig:ranking-identification}(a) shows that $\pi$PPO's advantage widens as generation progresses, reaching approximately 94\% pairwise accuracy compared with 78\% for PPO at the final position. We further form candidate sets from different trajectories of the same prompt at each of seven reasoning-node ranks, and test whether the critic identifies the state with the highest MC reference value. $\pi$PPO achieves higher chance-normalized top-1 accuracy at most evaluated positions, reaching approximately 50\% versus 15\% for PPO at the final position (Figure~\ref{fig:ranking-identification}(b)).

\begin{figure}[H]
\begin{minipage}[t]{0.49\linewidth}
    \vspace{0pt}
    \textbf{Reasoning progress decomposition.}
    Beyond value accuracy and state ranking, we examine whether
    the critic tracks local value changes across reasoning functions
    using the ThinkARM taxonomy
    \citep{li-etal-2026-schoenfelds}.
    This is assessed by correlating stepwise changes in critic predictions with changes in MC reference values and breaking down the results by reasoning function.
    Figure~\ref{fig:reasoning-function-value-tracking}
    shows higher correlations for $\pi$PPO overall and within
    Analyze and Implement, which together account for approximately
    90\% of the evaluated nodes.
    PPO's correlations are near zero in both functions, while $\pi$PPO reaches approximately 0.2.
    These results suggest improved tracking of local value changes
    during the dominant intermediate reasoning activities.
\end{minipage}\hfill
\begin{minipage}[t]{0.48\linewidth}
    \vspace{0pt}
    \centering
    \includegraphics[width=\linewidth]
    {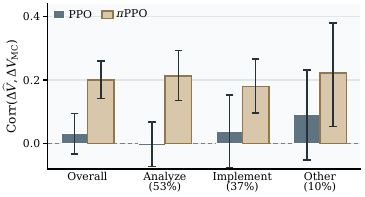}%
    \caption{Pearson correlation between predicted and
    MC-estimated value changes. Error bars show trajectory-cluster
    95\% CIs; percentages indicate the shares of evaluated nodes.}
    \label{fig:reasoning-function-value-tracking}
\end{minipage}
\end{figure}
\endgroup

\subsection{The role of contrastive evidence}
\label{sec:ablation}

\noindent
\begin{minipage}[t]{0.49\linewidth}
    \vspace{0pt}
    We test alternative forms of privileged information to identify which aspects contribute to the gains: removing correctness labels while retaining the selected references (\emph{w/o labels}), replacing opposite-outcome pairs with same-outcome references (\emph{w/ same polarity}), and using only the ground-truth final answer (\emph{w/ GT only}).
    Removing labels or opposite-outcome pairing erases most gains over baselines; the offline GT-only variant trails full $\pi$PPO by 4.7 and 3.1 points on Qwen3-4B and Qwen3-8B, respectively. These results link the gains to the contrastive organization of online trajectory evidence.
\end{minipage}\hfill
\begin{minipage}[t]{0.48\linewidth}
    \vspace{0pt}
    \centering
    \small
    \setlength{\tabcolsep}{3pt}
    \renewcommand{\arraystretch}{1.12}

    \providecommand{\ablhead}[1]{%
        \begin{tabular}{@{}c@{}}
            {\footnotesize\bfseries #1}\\[-0.8pt]
            {\fontsize{8.2}{9}\selectfont\texttt{Avg.}}
        \end{tabular}%
    }

    \begin{tabular*}{\linewidth}
        {@{\extracolsep{\fill}}lcc@{}}
        \toprule
        \multirow{2}{*}{\textbf{Method}}
        & \multicolumn{2}{c}{\textbf{Overall}} \\
        \cmidrule(lr){2-3}
        & \ablhead{Qwen3-4B}
        & \ablhead{Qwen3-8B} \\
        \midrule

        \textbf{$\bm{\pi}$PPO}
        & \textbf{50.3} & \textbf{51.6} \\

        \addlinespace[1pt]
        \quad w/o labels
        & 44.1 & 48.8 \\

        \quad w/ same polarity
        & 43.6 & 48.6 \\

        \quad w/ GT only
        & 45.6 & 48.5 \\

        \bottomrule
    \end{tabular*}
    \captionof{table}{Effects of correctness labels, reference polarity, and ground-truth-only context on average accuracy (\%) across five benchmarks.}
    \label{tab:context-ablation}
\end{minipage}
\par\medskip

\section{Related Work}

\textbf{RL for LLM through the lens of value estimation.}
RL algorithms used for LLM post-training are predominantly policy-gradient methods, with REINFORCE~\citep{reinforce} as their prototype. Prominent approaches in LLM post-training adopt either an actor--critic (e.g., PPO-based RLHF~\citep{schulman2017ppo,ouyang2022training}, VC-PPO~\citep{vcppo}, VAPO~\citep{yue2025vapo}) or critic-free formulation (e.g., bandit-like algorithms such as GRPO~\citep{shao2024grpo}, DAPO~\citep{yu2026dapo}, DR.GRPO~\citep{drgrpo}).
Despite differing in their use of a learned critic, both formulations can be examined through the lens of value estimation. From this perspective, a central question underlying their relative effectiveness is whether value estimation is sufficiently reliable to improve the bias--variance trade-off in policy-gradient updates.
$\pi$PPO addresses this question from the perspective of information available to the critic, introducing training-time privileged information to improve value estimation for policy optimization.

\textbf{Privileged information for policy optimization.}
Privileged information (PI) was formalized in learning using privileged information (LUPI), where a teacher supplies additional information available only during training to improve learning and generalization~\citep{vapnik2009new}; how such information is incorporated, however, remains flexible.
In LLM post-training, a direct approach is hint-based RL, adding PI from external experts into the input or supplied as a response prefix to guide sampling toward successful trajectories~\citep{chen2026nudging,qu2026pope}, training in an off-policy way, also causing a distributional mismatch between training and inference time.
Recent on-policy self-distillation uses PI to induce a stronger in-context distribution that supervises the policy on its own rollouts~\citep{zhao2026self,shenfeld2026self,tang-etal-2026-step}. This provides dense token-level supervision, but model-induced targets may not align with the original reward-maximization objective~\citep{kaur2026rethinking,nguyen2026privileged}.
$\pi$PPO places PI within the optimization process through critic-side value estimation, improving the information available for policy updates while preserving the actor's conditioning and the environmental reward objective.

\section{Conclusion}
In this paper, we recast value estimation in RLVR as an information problem, showing that a critic scoring each state in isolation can be inherently limited by its formulation, and that the evidence beyond target state can lower this floor.
Building on this view, we propose $\pi$PPO, which conditions the critic on a self-elicited, leave-target-out context of verified same-prompt siblings while preserving the actor's original conditioning and environment-defined reward objective, with no additional generation.
Across mathematical reasoning benchmarks with Qwen3-4B and Qwen3-8B backbones, $\pi$PPO outperforms its state-only actor--critic counterparts in critic quality, consistently achieving higher explained variance and lower value loss, while delivering the strongest end-to-end accuracy among all evaluated actor--critic and critic-free baselines---even staying ahead when we further exploit capacity asymmetry with a smaller critic model.
Overall, $\pi$PPO provides a complementary design view on developing more effective RL training algorithms for LLMs.

\bibliography{iclr2027_conference}
\bibliographystyle{iclr2027_conference}

\appendix
\clearpage

\section{Theoretical Proof}

\subsection{Value Estimation under Privileged Conditioning}
\label{app:privileged-risk-proof}

\begin{proof}[Proof of Proposition~\ref{prop:gap}]
All expectations are taken under the fixed joint distribution of $(s,\mathcal I,R)$ induced by the policy and the context construction. The binary return is square-integrable, and we consider only critics with finite squared error.

\paragraph{Step 1: Risk decomposition.}

For any fixed contextual critic, write
\[
R-V_\phi^+(s,\mathcal I)
= \bigl(R-V^{\pi,+}(s,\mathcal I)\bigr)
+ \bigl(V^{\pi,+}(s,\mathcal I)-V_\phi^+(s,\mathcal I)\bigr).
\]
By the definition of conditional expectation,
\[
\mathbb{E}[R-V^{\pi,+}(s,\mathcal I)\mid s,\mathcal I]=0.
\]
Since $V^{\pi,+}(s,\mathcal I)-V_\phi^+(s,\mathcal I)$ is a function
of $(s,\mathcal I)$, the cross term vanishes:
\[
\begin{aligned}
&\mathbb{E}\!\left[
\bigl(R-V^{\pi,+}(s,\mathcal I)\bigr)
\bigl(V^{\pi,+}(s,\mathcal I)-V_\phi^+(s,\mathcal I)\bigr)
\right] \\
&\quad=
\mathbb{E}\!\left[
\bigl(V^{\pi,+}(s,\mathcal I)-V_\phi^+(s,\mathcal I)\bigr)
\mathbb{E}[R-V^{\pi,+}(s,\mathcal I)\mid s,\mathcal I]
\right]=0.
\end{aligned}
\]

Squaring the decomposition and taking expectations therefore gives
\begin{equation}
\label{eq:contextual-risk-decomposition}
\begin{aligned}
\mathbb{E}[(R-V_\phi^+)^2]
&=\mathbb{E}_{s,\mathcal I}[\operatorname{Var}(R\mid s,\mathcal I)]
+\mathbb{E}[(V_\phi^+-V^{\pi,+})^2].
\end{aligned}
\end{equation}
The variance term does not depend on $\phi$, so it can be taken outside
the infimum over the shared critic parameters. This proves the first
equality in the proposition.

Similarly, conditioning on $s$ gives
\begin{equation}
\label{eq:state-only-risk-decomposition}
\begin{aligned}
\mathbb{E}[(R-V_\phi(s))^2]
&=\mathbb{E}_s[\operatorname{Var}(R\mid s)]
+\mathbb{E}[(V_\phi(s)-V^\pi(s))^2].
\end{aligned}
\end{equation}

Taking the infimum over $\phi$ establishes the final equality.

\paragraph{Step 2: Optimal risk comparison. } We next compare the two optimal risks. By assumption, for every
state-only parameter vector $\phi$ there exists a contextual parameter
vector $\psi$ (e.g., with the contribution of the additional context
$\mathcal I$ effectively set to zero) such that
\[
V_\psi^+(s,\mathcal I)=V_\phi(s)
\qquad\text{almost surely}.
\]

Consequently, for every $\phi$,
\[
\inf_\psi\mathbb{E}[(R-V_\psi^+(s,\mathcal I))^2]
\le\mathbb{E}[(R-V_\phi(s))^2].
\]
Taking the infimum on the right gives
\begin{equation}
\label{eq:critic-risk-inclusion}
\inf_\psi\mathbb{E}[(R-V_\psi^+(s,\mathcal I))^2]
\le\inf_\phi\mathbb{E}[(R-V_\phi(s))^2].
\end{equation}
Taking infima in Eqs.~(\ref{eq:contextual-risk-decomposition})
and~(\ref{eq:state-only-risk-decomposition}), whose variance terms
are independent of the critic parameters, and applying
Eq.~(\ref{eq:critic-risk-inclusion}) proves the proposition.
\end{proof}

\subsection{Success--Failure Contrast in Value Estimation}
\label{app:contrast-proof}

\begin{proof}[Proof of Proposition~\ref{prop:contrast}]
All expectations are taken under the fixed joint distribution of
$(s,R)$ specified in Proposition~\ref{prop:contrast}.
Since $R$ is binary and the critic has finite squared error,
all second moments below are finite.

\paragraph{Step 1: Outcome-conditioned risk decomposition.}
By the definition of $\Delta_\phi$ and the binary nature of $R$,
\[
\begin{aligned}
\mathbb{E}[R-V_\phi(s)\mid R]
&=(1-\Delta_\phi)R-\mathbb{E}[V_\phi(s)\mid R=0], \\
\operatorname{Var}(R-V_\phi(s)\mid R)
&=\operatorname{Var}(V_\phi(s)\mid R).
\end{aligned}
\]
Applying the law of total variance to $R-V_\phi(s)$ therefore gives
\begin{equation}
\label{eq:outcome-risk-decomposition}
\begin{aligned}
\mathbb{E}[(R-V_\phi(s))^2]
&=\bigl(\mathbb{E}[V_\phi(s)]-\mathbb{E}[R]\bigr)^2
+\operatorname{Var}(R-V_\phi(s)) \\
&=\bigl(\mathbb{E}[V_\phi(s)]-\mathbb{E}[R]\bigr)^2
+\mathbb{E}_R[\operatorname{Var}(V_\phi(s)\mid R)] \\
&\quad+\operatorname{Var}(R)(1-\Delta_\phi)^2.
\end{aligned}
\end{equation}

\paragraph{Step 2: Contrast lower bound.}
Let
\[
\varepsilon_\phi
=\bigl(\mathbb{E}[V_\phi(s)]-\mathbb{E}[R]\bigr)^2
+\mathbb{E}_R[\operatorname{Var}(V_\phi(s)\mid R)]
\ge 0.
\]
Equation~(\ref{eq:outcome-risk-decomposition}) then yields
\begin{equation}
\label{eq:contrast-risk-bound}
\begin{aligned}
\mathbb{E}[(R-V_\phi(s))^2]
&=\operatorname{Var}(R)(1-\Delta_\phi)^2+\varepsilon_\phi \\
&\ge\operatorname{Var}(R)(1-\Delta_\phi)^2.
\end{aligned}
\end{equation}
For any constant predictor $c$,
\[
\mathbb{E}[(R-c)^2]
=\operatorname{Var}(R)+(c-\mathbb{E}[R])^2,
\]
so the optimal constant prediction risk is $\operatorname{Var}(R)$,
attained at $c=\mathbb{E}[R]$.
If $\Delta_\phi\le 0$, then $(1-\Delta_\phi)^2\ge 1$, and
Eq.~(\ref{eq:contrast-risk-bound}) implies
$\mathbb{E}[(R-V_\phi(s))^2]\ge\operatorname{Var}(R)$.
Thus, any critic improving upon the optimal constant prediction risk
must satisfy $\Delta_\phi>0$.
\end{proof}%

\section{Implementation Details}
\label{app:implementation}
In this section, we provide the details of our main experiments in Section~\ref{sec:exp}.

\subsection{Training Details}
\label{app:train}

We build our implementation on the VeRL framework~\citep{verl}. All methods are trained for 250 steps with a batch size of 256, utilizing learning rates of $10^{-6}$ for the actor and $10^{-5}$ for the critic. Following DAPO and VAPO~\citep{yu2026dapo,yue2025vapo}, we adopt Clip-Higher with $\varepsilon_{\mathrm{low}}=0.2$ and $\varepsilon_{\mathrm{high}}=0.28$, omitting KL regularization.
For critic-based methods, we set $\gamma=\lambda=1.0$ and adopt fixed-policy value pretraining with Monte Carlo targets following VC-PPO~\citep{vcppo}.
Specifically, we perform 50 pretraining steps with the actor frozen, linearly warming up the critic learning rate from $10^{-7}$ to $10^{-5}$.

During policy training, we generate $G=8$ responses per prompt with temperature $1.0$, top-$p$ $1.0$, and top-$k$ disabled. For $\pi$PPO, references are selected randomly subject to the leave-target-out construction in Eq.~(\ref{eq:piac-context}). The maximum response length and the privileged-context truncation length are both set to 8192 tokens per sample. During value pretraining, we sample one response per prompt online, and set the batchsize to 2048 instead; $\pi$PPO obtains its reference context from a cache of the most recent verified rollouts from previous epochs. The remaining pretraining settings are shared with the standard critic. Algorithm~\ref{alg:pi-ppo} summarizes the overall procedure.

\paragraph{Baselines.}
\label{app:baselines}
All baselines share the common training, rollout, and evaluation settings described above. For a fair comparison, all methods use Clip-Higher and generate $G=8$ responses per prompt. For DAPO~\citep{yu2026dapo}, we implement dynamic sampling with a rejection pool size of 512, twice the training batch size. For VAPO~\citep{yue2025vapo}, we set the length-adaptive GAE coefficient $\alpha$ to $1.0$ (functioning only on the actor side).

\paragraph{Context ablations.}
\label{app:ablation-details}
All ablations modify the critic's privileged context and otherwise follow the common settings.
For \emph{w/o labels}, we preserve the selected reference pair, replace correctness headers with ordinal labels, and shuffle the reference order.
For \emph{w/ same polarity}, we supply two correct or two incorrect sibling references with their correctness labels, otherwise falling back to a contrastive pair.
For \emph{w/ GT only}, we remove all rollout references and provide only the verified ground-truth final answer.

\subsection{Evaluation and Analysis Details}
\label{app:evaluation}

We use the same vanilla evaluation prompts and decoding settings across methods for each backbone. Responses are sampled with temperature $0.6$, top-$p$ $0.95$, top-$k$ $20$, and a maximum length of 32768 tokens. We generate 16 responses per problem on AIME 2024, AIME 2025, HMMT February, and HMMT November, and 8 responses per problem on BeyondAIME. We report the average correctness across sampled responses as an estimate of pass@1, with Overall computed as the mean over the five benchmarks.

\textbf{Critic diagnostics.}
For critic diagnostics, we evaluate step-250 Qwen3-8B checkpoints on 128 prompts per method and estimate values from 256 MC continuations at seven semantic prefixes, with completed responses capped at 8192 tokens. Pairwise ranking uses 20 prefix positions; top-1 identification retains candidate sets whose MC-value range exceeds twice the estimated uncertainty.

\begin{algorithm}[t]
\caption{$\pi$PPO: Privileged-Information PPO}
\label{alg:pi-ppo}
\DontPrintSemicolon

\KwIn{
Actor $\pi_\theta$; warmed-up privileged critic $V_\phi^+$;
prompt dataset $\mathcal D$ with ground-truth answers;
verifier $R$; group size $G$;
discount $\gamma$; GAE parameter $\lambda$
}

\For{each training iteration}{
    $\theta_{\mathrm{old}} \leftarrow \theta$\;
    Sample a prompt batch $\mathcal B \subset \mathcal D$\;

    \ForEach{prompt $\boldsymbol x \in \mathcal B$}{
        Retrieve its ground-truth answer $\boldsymbol y^\star$\;

        \tcp{Collect and verify same-prompt rollouts}
        Independently sample
        $\{\boldsymbol y^{(i)}\}_{i=1}^{G}
        \sim \pi_{\theta_{\mathrm{old}}}(\cdot\mid\boldsymbol x)$\;
        Evaluate $R(\boldsymbol x,\boldsymbol y^{(i)})\in\{0,1\}$
        and construct token rewards $\{r_t^{(i)}\}_{t=1}^{T_i}$,
        where $T_i=|\boldsymbol y^{(i)}|$\;

        \For{$i=1,\ldots,G$}{
            \tcp{Exclude the target before selecting references}
            $\mathcal C_i^+ \leftarrow
            \{\boldsymbol y^{(j)}:
            j\neq i,\ R(\boldsymbol x,\boldsymbol y^{(j)})=1\}$\;
            $\mathcal C_i^- \leftarrow
            \{\boldsymbol y^{(j)}:
            j\neq i,\ R(\boldsymbol x,\boldsymbol y^{(j)})=0\}$\;

            \uIf{$\mathcal C_i^+\neq\varnothing$
                 \textnormal{ and } $\mathcal C_i^-\neq\varnothing$}{
                $\mathcal I^{(i)} \leftarrow
                \mathcal S_1(\mathcal C_i^+)
                \Vert \mathcal S_1(\mathcal C_i^-)$\;
            }
            \uElseIf{$\mathcal C_i^-=\varnothing$}{
                $\mathcal I^{(i)} \leftarrow
                \mathcal S_2(\mathcal C_i^+)$\;
            }
            \Else{
                $\mathcal I^{(i)} \leftarrow
                \mathcal S_2(\mathcal C_i^-)
                \Vert [\boldsymbol y^\star]$\;
            }

            \tcp{Estimate values under a fixed context}
            $s_t^{(i)}\leftarrow
            (\boldsymbol x,\boldsymbol y_{<t}^{(i)})$\;
            Compute and store
            $\widehat V_{i,t}\leftarrow
            V_\phi^+(s_t^{(i)},\mathcal I^{(i)})$,
            with zero value at terminal states\;

            \For{$t=T_i,\ldots,1$}{
                $\delta_t^{+(i)}\leftarrow
                r_t^{(i)}+\gamma\widehat V_{i,t+1}
                -\widehat V_{i,t}$\;
                $\widehat A_t^{+(i)}\leftarrow
                \delta_t^{+(i)}
                +\gamma\lambda\widehat A_{t+1}^{+(i)}$,
                with $\widehat A_{T_i+1}^{+(i)}=0$\;
            }
            Compute and store value-regression targets
            $\{\widehat G_{i,t}\}_{t=1}^{T_i}$
            using the PPO return-target rule\;
        }
    }
    \For{each PPO optimization epoch}{
        Update $\phi$ by value regression on
        $(s_t^{(i)},\mathcal I^{(i)},\widehat G_{i,t})$\;
        Update $\theta$ using the clipped PPO objective
        with $\widehat A_t^{+(i)}$ and
        $\displaystyle
        \rho_t^{(i)}(\theta)=
        \frac{\pi_\theta(y_t^{(i)}\mid
        \boldsymbol x,\boldsymbol y_{<t}^{(i)})}
        {\pi_{\theta_{\mathrm{old}}}(y_t^{(i)}\mid
        \boldsymbol x,\boldsymbol y_{<t}^{(i)})}$\;
    }
}
\end{algorithm}
\section{Prompt Template for $\pi$PPO}

\begin{tcolorbox}[
  enhanced,
  sharp corners,
  width=\textwidth,
  boxrule=.6pt,
  colback=white,
  colframe=black!55,
  coltitle=white,
  colbacktitle=black!68,
  fontupper=\small,
  left=9pt,
  right=9pt,
  top=7pt,
  bottom=7pt,
  title={\textbf{Prompt Template used in $\pi$PPO}\hfill\normalfont\small}
]

\noindent{\bfseries\texttt{<|im\_start|>user}}\par
\vspace{.7mm}
\noindent\hspace*{1.5em}\cmt{[shared source: original user turn $\rightarrow$ actor, $V_\phi$, $V_\phi^+$]}\par
\vspace{1.6mm}
\noindent
Let $x$ and $y$ be real numbers with $x+y=25$ and $xy=126$. Find $|x-y|$. \\
Please reason step by step, and put your final answer within \verb|\boxed{}|.

\vspace{2.4mm}
\begin{privblock}
\noindent\hspace*{1.2em}\cmt{[critic-only source: labeled reference attempts $\rightarrow$ $V^+_\phi$ only]}\par
\vspace{1.5mm}
Below are reference attempts at the problem. Each one is explicitly labeled with whether it reached the correct final answer. They are backgrounded and are NOT part of the response.

\vspace{2.5mm}
\lbl{[Reference attempt --- CORRECT]:}
\begin{pistep}
  \item The sum and the product are given, so let $s=x+y=25$ and $p=xy=126$.
  \item I want $|x-y|$, and the clean route is to obtain $(x-y)^2$ from $s$ and $p$.
  \item Since $(x-y)^2=x^2-2xy+y^2$ and $(x+y)^2=x^2+2xy+y^2$, subtracting gives $(x-y)^2=(x+y)^2-4xy$.
  \item Substituting, $(x-y)^2=625-4\cdot 126=625-504=121$.
  \item Therefore $|x-y|=\sqrt{121}=11$.
  \item Check: the numbers solve $t^2-25t+126=0$, so they are $18$ and $7$; their sum is $25$ and their product is $126$, both as required.
  \item The final answer is $\boxed{11}$.
\end{pistep}

\vspace{2.5mm}
\lbl{[Reference attempt --- INCORRECT]:}
\begin{pistep}
  \item The sum and the product are given, so let $s=x+y=25$ and $p=xy=126$.
  \item I want $|x-y|$, and the clean route is to obtain $(x-y)^2$ from $s$ and $p$.
  \item Expanding the square of the difference, $(x-y)^2=(x+y)^2-2xy$.
  \item Substituting, $(x-y)^2=625-2\cdot 126=625-252=373$.
  \item Therefore $|x-y|=\sqrt{373}\approx 19.31$.
  \item Check: the numbers are then about $22.16$ and $2.84$, and $22.16+2.84=25$, so the sum comes out right.
  \item The final answer is $\boxed{\sqrt{373}}$.
\end{pistep}
\end{privblock}

\vspace{2.4mm}
{\bfseries\texttt{<|im\_end|>}} \\[.6mm]
\noindent{\bfseries\texttt{<|im\_start|>assistant}}\par
\vspace{.7mm}
\noindent\hspace*{1.5em}\cmt{[shared source: actor-sampled completion $\rightarrow$ $V_\phi$, $V_\phi^+$]}\par
\vspace{1.6mm}
\begin{pistep}
  \item Let $x$ and $y$ be the two numbers, so $x+y=25$ and $xy=126$.
  \item Then $x$ and $y$ are the roots of $t^2-25t+126=0$.
  \item Its discriminant is $25^2-4\cdot 126=625-504=121$.
  \item Hence $t=\frac{25\pm\sqrt{121}}{2}=\frac{25\pm 11}{2}$, giving $t=18$ and $t=7$.
  \item So $|x-y|=18-7=11$.
  \item The final answer is $\boxed{11}$.
\end{pistep}
\vspace{.8mm}
{\bfseries\texttt{<|im\_end|>}}
\end{tcolorbox}

\begingroup
\captionof{table}{The actor and the state-only critic $V_\phi$ share the original user turn and actor-generated completion, whereas the privileged critic $V_\phi^+$ receives the same completion with a labeled block of additional information appended to the user turn. The two reference attempts in this block are numbered step by step and agree verbatim through step $(2)$, but diverge at step $(3)$ in their expansion of $(x-y)^2$, using $-4xy$ and $-2xy$, respectively, forming a direct process-level contrast.}
\label{tab:pi_prompt}
\endgroup
\section{Additional Results}

\subsection{Evaluation on OOD benchmarks}
\label{app:ood}

\makeatletter
\@ifundefined{bcol}{\newlength{\bcol}}{}
\makeatother
\providecommand{\bname}[1]{\mbox{\footnotesize\bfseries #1}}
\providecommand{\kavg}{}
\renewcommand{\kavg}[1]{\mbox{\fontsize{8.2}{9}\selectfont\texttt{#1}}}
\providecommand{\benchhead}{}
\renewcommand{\benchhead}[2]{%
  \begin{tabular}{@{}c@{}}
    \makebox[\bcol][c]{\bname{#1}}\\[-0.8pt]
    \makebox[\bcol][c]{\kavg{#2}}%
  \end{tabular}%
}
\providecommand{\subarrow}{}
\renewcommand{\subarrow}{\hspace{0.5em}\rotatebox[origin=c]{180}{$\Lsh$}}

\begin{table}[t]
\centering
\caption{Out-of-domain accuracy (\%) of policies trained on DAPO-17K.
``\texttt{Avg@k}'' denotes mean accuracy over $k$ random generations (i.e., pass@1);
\textbf{Overall} is the mean over the two benchmarks.
Asym: the method above with a smaller critic (0.6B for 4B, 1.7B for 8B).
Best result within each model size is in bold.}
\label{tab:ood}

\small
\setlength{\tabcolsep}{3pt}
\renewcommand{\arraystretch}{1.12}
\setlength{\bcol}{50pt}

\begin{tabular}{@{} l @{\hspace{8pt}} *{6}{w{c}{\bcol}} @{}}
    \toprule

    \multirow{2}{*}{\textbf{Method}}
    & \multicolumn{3}{c}{\textbf{Qwen3-4B}}
    & \multicolumn{3}{c}{\textbf{Qwen3-8B}} \\
    \cmidrule(lr){2-4} \cmidrule(lr){5-7}
    & \benchhead{GPQA}{Avg@8} & \benchhead{MMLU}{Avg@1} & \benchhead{Overall}{Avg.}
    & \benchhead{GPQA}{Avg@8} & \benchhead{MMLU}{Avg@1} & \benchhead{Overall}{Avg.} \\
    \midrule

        Initial policy
        & 45.7 & 60.8 & 53.3 & 51.0 & 66.1 & 58.6 \\

        \addlinespace[1pt]
        \multicolumn{7}{@{}l@{}}{\emph{Critic-free}} \\

        \quad GRPO
        & \textbf{52.9} & 67.3 & 60.1 & 58.3 & 71.6 & 65.0 \\

        \quad DAPO
        & 52.8 & 67.7 & 60.3 & 56.0 & \textbf{72.0} & 64.0 \\

        \addlinespace[1pt]
        \multicolumn{7}{@{}l@{}}{\emph{Actor--critic}} \\

        \quad PPO
        & 51.5 & 66.3 & 58.9 & 56.3 & 71.3 & 63.8 \\

        \quad \subarrow\ Asym
        & 52.7 & 66.8 & 59.8 & 57.9 & 71.4 & 64.7 \\

        \quad VAPO
        & 51.4 & 66.0 & 58.7 & 56.9 & 69.7 & 63.3 \\

        \quad \textbf{$\bm{\pi}$PPO (ours)}
        & 52.8 & \textbf{68.2} & \textbf{60.5} & 57.7 & 71.7 & 64.7 \\

        \quad \subarrow\ Asym
        & 52.6 & 66.6 & 59.6 & \textbf{58.8} & 71.5 & \textbf{65.2} \\

        \bottomrule
\end{tabular}
\end{table}

We additionally evaluated our method on representative multi-task benchmarks spanning diverse domains in science, humanities, and professional knowledge: GPQA-Diamond~\citep{gpqa} and MMLU-Pro~\citep{mmlu-pro}, for examining the out-of-distribution generalization.

As shown in Table~\ref{tab:ood}, all RL-trained policies outperform their initial counterparts. At the 4B scale, $\pi$PPO achieves the best Overall score of 60.5. At the 8B scale, $\pi$PPO reaches 64.7, its asymmetric variant achieves the best 8B Overall score of 65.2 using only a 1.7B critic. These results suggest that the in-domain advantages of $\pi$PPO may carry over to other related domains.

\subsection{Additional Comparison with self-distillation methods}
\label{app:opsd}

\makeatletter
\@ifundefined{bcol}{\newlength{\bcol}}{}
\makeatother
\setlength{\bcol}{50pt}
\providecommand{\bname}[1]{\mbox{\footnotesize\bfseries #1}}
\providecommand{\kavg}{}
\renewcommand{\kavg}[1]{\mbox{\fontsize{8.2}{9}\selectfont\texttt{#1}}}
\providecommand{\benchhead}{}
\renewcommand{\benchhead}[2]{%
  \begin{tabular}{@{}c@{}}
    \makebox[\bcol][c]{\bname{#1}}\\[-0.8pt]
    \makebox[\bcol][c]{\kavg{#2}}%
  \end{tabular}%
}

\begin{table}[t]
\centering
\caption{Comparison with RLSD on mathematical reasoning benchmarks. \textbf{Overall} averages the five benchmark scores; bold denotes the best result within each model size.}
\label{tab:rlsd}

\small
\setlength{\tabcolsep}{3pt}
\renewcommand{\arraystretch}{1.12}

\begin{tabular}{@{} l @{\hspace{8pt}} *{6}{w{c}{\bcol}} @{}}
    \toprule

    \multirow{2}{*}{\textbf{Method}}
    & \multicolumn{6}{c}{\textbf{Mathematical Reasoning}} \\
    \cmidrule(lr){2-7}
    & \benchhead{AIME 24}{Avg@16}
    & \benchhead{AIME 25}{Avg@16}
    & \benchhead{Beyond AIME}{Avg@8}
    & \benchhead{HMMT\textsubscript{Feb}}{Avg@16}
    & \benchhead{HMMT\textsubscript{Nov}}{Avg@16}
    & \benchhead{Overall}{Avg.} \\
    \midrule

        \multicolumn{7}{@{}l@{}}{\textbf{(a) Qwen3-4B}} \\

        Initial policy
        & 21.7 & 18.9 & 10.8 & 12.3 & 7.3 & 14.2 \\

        RLSD
        & 57.1 & 47.9 & 31.5 & 31.5 & 38.8 & 41.4 \\

        \textbf{$\bm{\pi}$PPO (ours)}
        & \textbf{66.7} & \textbf{61.5} & \textbf{39.0}
        & \textbf{37.5} & \textbf{46.9} & \textbf{50.3} \\

        \addlinespace[2pt]
        \midrule
        \multicolumn{7}{@{}l@{}}{\textbf{(b) Qwen3-8B}} \\

        Initial policy
        & 25.0 & 20.2 & 13.1 & 11.5 & 10.0 & 16.0 \\

        RLSD
        & 67.5 & 50.4 & 35.3 & 28.3 & 38.8 & 44.1 \\

        \textbf{$\bm{\pi}$PPO (ours)}
        & \textbf{73.5} & \textbf{61.5} & \textbf{40.0}
        & \textbf{34.4} & \textbf{48.5} & \textbf{51.6} \\

        \bottomrule
\end{tabular}
\end{table}

We additionally compare with OPSD~\citep{zhao2026self} and its RL-integrated variants, RLSD~\citep{yang2026self} and RLCSD~\citep{pan2026rlcsd}, since these methods can be viewed in the same general way as ours: leveraging privileged information to provide additional supervision or shape token-level rewards.

We follow the original implementations and adapt these methods to the training setup used in our main experiments. The reference configuration uses a batch size of 256, an actor learning rate of $10^{-6}$, and a budget of 250 policy updates. We sample $G=8$ responses per prompt with temperature $1.0$, top-$p$ $1.0$, disabled top-$k$ filtering, and a maximum response length of 8192 tokens. We explore multiple configurations around this setup to adapt the methods for better training stability and effectiveness.

Across the tested configurations, we consistently observe training oscillations and collapse for OPSD and RLCSD, with no effective gain observed compared to the initial model. Among these three baselines, only RLSD achieves reasonable downstream performance under our experimental setting, which follows their original setups.

\end{document}